\documentclass{article}

\PassOptionsToPackage{numbers, compress}{natbib}
\usepackage[preprint]{neurips_2026}

\usepackage[utf8]{inputenc} % allow utf-8 input
\usepackage[T1]{fontenc}    % use 8-bit T1 fonts
\usepackage{hyperref}       % hyperlinks
\usepackage{url}            % simple URL typesetting
\usepackage{booktabs}       % professional-quality tables
\usepackage{amsfonts}       % blackboard math symbols
\usepackage{nicefrac}       % compact symbols for 1/2, etc.
\usepackage{microtype}      % microtypography
\usepackage{xcolor}         % colors

\usepackage{graphicx}
\usepackage{amsmath}
\usepackage{amssymb}
\usepackage{mathtools}
\usepackage{amsthm}

\usepackage{algorithm}
\usepackage{algorithmic}
\usepackage{multirow}
\usepackage{wrapfig}
\usepackage[position=bottom]{subfig}
\usepackage{pifont}
\usepackage{colortbl}
\usepackage{tcolorbox}
\usepackage{tablefootnote}
\usepackage{enumitem}
\usepackage{thmtools}
\usepackage{thm-restate}

\usepackage[capitalize,noabbrev]{cleveref}

\theoremstyle{plain}

\theoremstyle{definition}

\theoremstyle{remark}

\renewcommand{\vec}[1]{\mathbf{#1}}

\title{Accelerating Attention with Basis Decomposition}

\author{%
  Jialin Zhao \\
  Department of Computer Science\\
  Tsinghua University\\
  \texttt{jialin.zhao97@gmail.com} \\
}

\begin{document}

\definecolor{customcolor}{RGB}{211, 211, 211} % Light gray
\definecolor{MyDarkGreen}{rgb}{0.0, 0.5, 0.0}

\maketitle

\begin{abstract}
  Attention is a core operation in large language models (LLMs). We present BD Attention (\textbf{BDA}), a \textit{lossless algorithmic reformulation} of attention. BDA is enabled by a simple matrix identity from Basis Decomposition (\textbf{BD}), which restructures multi-head projections into a compact form while preserving exact outputs. Unlike I/O-aware system optimizations such as FlashAttention, BDA provides a mathematically guaranteed acceleration that is architecture-agnostic. On DeepSeek-V2-Lite (16B, FP16), BDA requires only \textbf{4s} of offline preparation \textbf{with no retraining required} and, on modern GPUs, achieves \textbf{34\% faster} key/value projections and \textbf{25\% smaller} weights, while increasing perplexity (PPL) by just \textbf{0.02\%} (FP16) or \textbf{0.0004\%} (FP32)—a negligible effect on model performance. These results position BDA as a theoretically exact method for lossless attention acceleration that is complementary to existing engineering-level optimizations. Our code is available at \url{https://anonymous.4open.science/r/Basis-decomp-57B8}.
\end{abstract}

\section{Introduction}
\label{sec:intro}

Attention \citep{vaswani2017attention} has emerged as the fundamental building block in large language models (LLMs) and vision-language models (VLMs), enabling them to scale effectively across diverse tasks. However, the cost of multi-head attention (MHA) in both computation and memory makes it a major efficiency bottleneck. Existing acceleration techniques can be broadly divided into two categories. On the one hand, \textit{lossless system-level optimizations}, such as FlashAttention~\citep{dao2022flashattention}, improve I/O efficiency by reordering memory access and fusing kernels, but their gains are hardware-specific and do not reduce the number of arithmetic operations or parameters. On the other hand, \textit{approximate algorithmic approaches}, including linear attention~\citep{katharopoulos2020transformers,choromanski2021rethinking}, sparse attention~\citep{child2019generating,beltagy2020longformer}, pruning~\citep{ma2023llmpruner,frantar2023sparsegpt}, and quantization~\citep{frantar2023optq,MLSYS2024_42a452cb}, reduce complexity or storage but inevitably trade off accuracy and often require retraining or calibration.

In this work, we take a different path and propose \textbf{BD Attention (BDA)}, a \textit{lossless algorithmic reformulation} of attention. BDA is enabled by a simple yet general matrix identity from \textbf{Basis Decomposition (BD)} to restructure MHA projections into a compact form, thereby reducing parameters and arithmetic operations while preserving exact outputs. This algorithmic perspective complements existing I/O-aware optimizations.

We validate BDA across multiple scenarios. On DeepSeek-V2-Lite (16B, FP16), BDA requires only \textbf{4s} of offline preparation with no retraining and achieves \textbf{34\% faster} key/value projections and \textbf{25\% smaller} weights on modern GPUs, with perplexity (PPL) change at the negligible level of \textbf{0.02\%}. Beyond inference, training experiments show that BDA achieves BLEU scores comparable to MHA without any hyperparameter adjustment. Furthermore, when applied on top of low-rank pruned models, BD reduces memory by \textbf{16.5\%} and improves throughput by \textbf{17.2\%}, demonstrating its compatibility with existing compression techniques. The magnitude of this speedup depends on the head-to-embedding dimension ratio, and is most pronounced in MLA-style architectures such as DeepSeek-V2/V3.

Overall, BDA establishes a new algorithmic foundation for lossless attention acceleration, uniting theoretical guarantees with practical speedups on real hardware.

We highlight two main contributions of this study:

\begin{itemize}
    \item \textbf{Theoretical foundation.} We propose \textbf{Basis Decomposition (BD)}, a general matrix decomposition that guarantees hardware-friendly reconstruction with probability 1 under assumptions that are naturally satisfied by neural network weight matrices (Theorem~\ref{thm:full-rank-random}).
    \item \textbf{Practical acceleration.} We apply BD to attention and low-rank layers,
    showing that it reduces parameters and arithmetic operations without degrading model quality.
    We demonstrate consistent inference acceleration with both a standard PyTorch implementation on CPU
    and a fused Triton kernel on GPU, achieving near-theoretical speedups in practice.

\end{itemize}

\begin{table}[t]
\centering
\caption{\textbf{Comparison of acceleration techniques for Transformers.}}
\label{tab:method_comp}
\begin{tabular}{lccccc}
\toprule
Method & Lossless & (Re)Train & CPU Speedup & GPU Speedup & \# Params \\
\midrule
Flash Attention & \textcolor{MyDarkGreen}{\ding{51}} & \ding{55}             & \ding{55}  & \textcolor{MyDarkGreen}{\ding{51}} & Same  \\
Linear / Sparse Attention& \ding{55}  & \textcolor{MyDarkGreen}{\ding{51}}            & \textcolor{MyDarkGreen}{\ding{51}} & \textcolor{MyDarkGreen}{\ding{51}} & Same  \\
Pruning / Quantization         & \ding{55}  & Partial  & \textcolor{MyDarkGreen}{\ding{51}} & \textcolor{MyDarkGreen}{\ding{51}} & Reduce \\
\cellcolor{customcolor}BD Attention     & \cellcolor{customcolor}\textcolor{MyDarkGreen}{\ding{51}} & \cellcolor{customcolor}\ding{55}             & \cellcolor{customcolor}\textcolor{MyDarkGreen}{\ding{51}} & \cellcolor{customcolor}\textcolor{MyDarkGreen}{\ding{51}} & \cellcolor{customcolor}Reduce \\
\bottomrule
\end{tabular}
\vspace{-1em}
\end{table}

\section{Related Work}
\label{sec:related}

Research on accelerating attention broadly falls into two categories: (i) \textit{lossless} system-level optimizations that preserve exact outputs, and (ii) \textit{approximate} algorithms that trade accuracy for efficiency. Table~\ref{tab:method_comp} summarizes representative approaches.

\paragraph{Lossless methods.}
FlashAttention and its successors~\citep{dao2022flashattention,dao2023flashattention,NEURIPS2024_7ede97c3,kwon2023efficient} reduce memory traffic by fusing attention kernels and tiling reads/writes to on-chip SRAM, achieving substantial GPU speedups without altering outputs. These approaches are \textit{lossless} but depend heavily on GPU memory hierarchies and I/O architectures. Crucially, FlashAttention and BDA target disjoint steps of attention: FlashAttention optimizes the softmax computation $\mathrm{softmax}(\vec{Q}\vec{K}^\top)\vec{V}$ (Algorithm~\ref{alg:mha_bd_inference}, step 5), whereas BDA reduces FLOPs in the projection step $\vec{X}\!\to\!\vec{K},\vec{V}$ (Algorithm~\ref{alg:mha_bd_inference}, steps 2--3). The two are therefore complementary and can be combined.

\paragraph{Approximate methods.}
In contrast, many lines of work accelerate attention or reduce memory by introducing approximations.

\begin{itemize}
  \item \textbf{Linear-attention} replaces the quadratic softmax kernel with linear-time kernel approximations~\citep{wang2020linformer,katharopoulos2020transformers,choromanski2021rethinking,schlag2021linear,chen2021skyformer}.

  \item \textbf{Sparse-attention} imposes predefined or learned sparsity patterns, e.g., Sparse Transformers~\citep{child2019generating}, Longformer~\citep{beltagy2020longformer}, Reformer~\citep{Kitaev2020Reformer}, and sketching-based sampling methods~\citep{chen2022sketching}.

  \item \textbf{Pruning} removes parameters based on importance criteria.
  \begin{itemize}
    \item \textbf{Unstructured pruning} targets individual weights~\citep{frantar2023sparsegpt,sun2024a,zhang2024plugandplay}, effective in reducing parameters but hard to accelerate on GPUs.
    \item \textbf{Semi-structured pruning} (e.g., $N\!:\!M$ patterns)~\citep{mishra2021accelerating} enables GPU acceleration on NVIDIA’s Ampere GPUs but restricts density flexibility.
    \item  \textbf{Structured pruning} removes entire channels/heads~\citep{ma2023llmpruner,ouderaa2024the,ashkboos2024slicegpt}, more deployment-friendly but often with higher accuracy loss.
    \item \textbf{Low-rank pruning} approximates weight matrices by decomposing them into lower-rank components~\citep{yuan2023asvd,wang2024svd,zhao2025pivoting}, reducing both parameters and compute. While hardware-friendly, it introduces approximation error unless the true rank is low.
  \end{itemize}

  \item \textbf{Quantization} compresses parameters/activations into low-precision formats~\citep{frantar2023optq,MLSYS2024_42a452cb,xiao2023smoothquant,lin2024duquant}, achieving strong memory savings but inevitably introducing errors at low bit-widths.
\end{itemize}

These methods can deliver large efficiency gains but are inherently \textit{not lossless}, and usually require careful finetuning to mitigate quality drops.

\section{Method}
\label{sec:method}

\begin{figure}[t]
		\centering

  \includegraphics[width=1\linewidth]{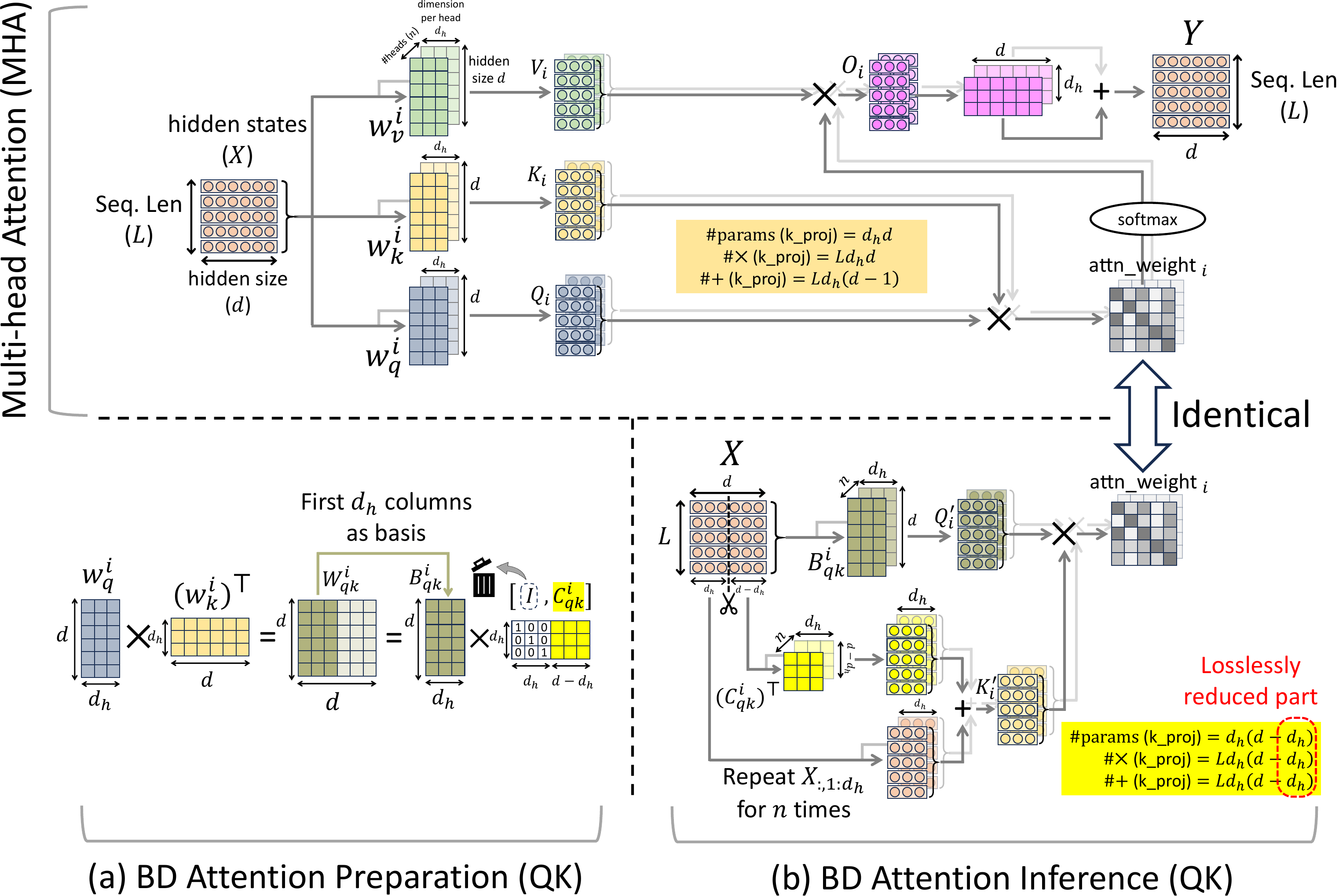}

  \caption{Illustration of BD Attention (\textbf{BDA}) using the QK projection as an example (VO is analogous).
BDA consists of two stages:
(a) \textbf{BD Attention Preparation} (Algorithm \ref{alg:mha_bd_prepare}), performed offline \textit{once} during model deployment, where the projection matrices are transformed via Basis Decomposition;
(b) \textbf{BD Attention Inference} (Algorithm \ref{alg:mha_bd_inference}) saves $d_h / d$ in both parameters and computation, while preserving exact outputs.}

		\label{fig:bd_illustration}

    \vspace{-1em}

\end{figure}

\subsection{Basis Decomposition}

Given two rectangular matrices \(\Vec{U} \in \mathbb{R}^{m \times r}\) and \(\Vec{V}^{\top} \in \mathbb{R}^{r \times n}\), where \(r < \min(m, n)\), their product matrix \(\vec{W} = \Vec{U}\Vec{V}^{\top}\) has rank at most \(r\). Our goal is to develop a general decomposition method that expresses \(\vec{W}\) as \(\vec{W} = f(\vec{M}_1, \ldots, \vec{M}_k)\), where the function \(f(\vec{M}_1, \ldots, \vec{M}_k)\) has lower computational cost than the original multiplication \(\Vec{U}\Vec{V}^{\top}\).

We assume \(\vec{W}\) has rank \(r\) without loss of generality. In this case, there exist exactly \(r\) linearly independent rows in \(\vec{W}\). We define a matrix \(\vec{B} = [\vec{b}_1, \ldots, \vec{b}_r]^{\top} \in \mathbb{R}^{r \times n}\), where each \(\vec{b}_j\) is one of these linearly independent rows. Thus, \(\vec{B}\) forms a basis of the row space of \(\vec{W}\). For any row \(\vec{w}_i \in \vec{W} \setminus \vec{B}\), where \(\vec{w}_i \in \mathbb{R}^{n}\), it can be written as a linear combination of the basis vectors in \(\vec{B}\):

\begin{equation}
\label{equ:linear_combination}
\vec{w}_i = \sum_{j=1}^r c_{ij} \vec{b}_j,
\end{equation}

which follows directly from the fact that the row space of \(\vec{W}\) is spanned by the \(r\) basis vectors in \(\vec{B}\). Collecting all such coefficients \(c_{ij}\) forms a coefficient matrix \(\vec{C} \in \mathbb{R}^{(m - r) \times r}\), where each row contains the weights for reconstructing a non-basis row of \(\vec{W}\) as a linear combination of the basis vectors in \(\vec{B}\). With \(\vec{B}\) and \(\vec{C}\), the original matrix \(\vec{W}\) can be fully reconstructed.

This forms an alternative representation of the low-rank matrix \(\vec{W}\), distinct from the traditional low-rank multiplication \(\vec{U} \vec{V}^{\top}\). We refer to this representation as \textbf{Basis Decomposition (BD)}, encompassing both the decomposition \(\vec{W} \rightarrow (\vec{B}, \vec{C})\) and its reconstruction \((\vec{B}, \vec{C}) \rightarrow \vec{W}\). We use the term \textit{BD} for general reference, and specify \textit{BD decomposition / reconstruction} only when emphasizing the decomposition/reconstruction process.

\paragraph{Memory cost of BD.} BD stores two matrices: the basis matrix \(\vec{B} \in \mathbb{R}^{r \times n}\) and the coefficient matrix \(\vec{C} \in \mathbb{R}^{(m - r) \times r}\). The total memory cost is \(r(m + n - r)\), which is strictly smaller than the full matrix size \(mn\) and low rank matrices size \(r(m+n)\) for any \(r < \min(m, n)\). In contrast, the traditional low-rank representation \(\vec{W} = \vec{U} \vec{V}^{\top}\) requires \(r(m + n)\) parameters and is only more compact than the full matrix when \(r < \frac{mn}{m + n}\).

\paragraph{Computational cost of BD reconstruction.} Reconstructing \(\vec{W}\) from BD involves computing \(\vec{C} \vec{B}\) and inserting the \(r\) basis rows. This requires \(2r(m - r)n\) floating-point operations (FLOPs). In comparison, the traditional reconstruction \(\vec{U} \vec{V}^{\top}\) costs \(2rmn\) FLOPs. Therefore, BD reconstruction is computationally more efficient for any \(r < \min(m, n)\).

\subsection{Selection of Basis}

Although \(\vec{W} = \vec{U}\vec{V}^\top\) has rank exactly \(r\) by construction, this does not immediately imply that an \emph{arbitrary} \(r\)-row subset of \(\vec{W}\) is a valid basis: the chosen rows must themselves be linearly independent. Equivalently, the selected \(r\)-row submatrix \(\vec{B}\in\mathbb{R}^{r\times n}\) must contain at least one invertible \(r\times r\) sub-block. The following theorem shows that, for weights drawn from a continuous distribution, this holds with probability 1 for \emph{any} subset of \(r\) rows.

\begin{restatable}[Almost Sure Full Rank of Random Matrices]{theorem}{fullRankRandomTheorem}
\label{thm:full-rank-random}
Let $\vec{W}$ be an $r \times r$ real random matrix. Suppose the entries of $\vec{W}$
are drawn from a probability measure $\mu$ on $\mathbb{R}^{r^2}$
that is absolutely continuous with respect to the Lebesgue measure $\lambda$.
Then $W$ has full rank (\(\mathrm{rank}(W)=r\)) with probability 1.
\end{restatable}

Theorem~\ref{thm:full-rank-random} (proof in Appendix \ref{sec:proof_full_rank_random}) states that if a matrix \(\vec{M} \in \mathbb{R}^{r \times r}\) has entries drawn from a distribution that is absolutely continuous with respect to the Lebesgue measure, then \(\vec{M}\) is full rank with probability 1. In practical terms, this applies to any matrix whose entries follow a continuous distribution, including weights perturbed by any non-degenerate continuous noise.

Therefore, when \(\vec{W}\) is a product of weight matrices \(\Vec{U}\Vec{V}^\top\) whose entries follow a continuous distribution, any \(r\) rows of \(\vec{W}\) form a submatrix \(\vec{B}\in\mathbb{R}^{r\times n}\) whose entries are also continuous. By Theorem~\ref{thm:full-rank-random}, every \(r\times r\) sub-block of \(\vec{B}\) (obtained by choosing any \(r\) columns) is invertible with probability 1, so the \(r\) rows of \(\vec{B}\) are linearly independent. We thus \textbf{freely choose} any \(r\) rows (or symmetrically, any \(r\) columns) as a valid basis for BD reconstruction, without explicit rank analysis or pivoting. This situation is common in practice—for example, when \(\Vec{U}\) and \(\Vec{V}\) are weight matrices learned via stochastic gradient descent (SGD), whose stochastic dynamics make the resulting weights samples from a continuous distribution rather than deterministic values~\citep{mandt2017stochastic, zhou2020towards}.

Notably, PIFA~\citep{zhao2025pivoting} can be regarded as a special case of BD: it selects basis rows via QR factorization with column pivoting, approximating the most numerically independent directions~\citep{businger1971linear}. While this is useful in rare rank-deficient or ill-conditioned settings, such guarantees are unnecessary in typical neural networks where the stochasticity of training ensures full rank with probability 1. Hence, BD offers a more flexible and efficient basis selection strategy.

In particular, we find that choosing contiguous rows/columns, such as the first‑\(r\) or last‑\(r\) rows/columns, offers significant efficiency advantages. It minimizes the I/O overhead during reconstruction by avoiding scattered memory writes/reads on modern hardware such as GPUs. We adopt this strategy as the standard basis selection method for BD in practice.

Let $\vec{I}\in\mathbb{R}^{r\times r}$ denote the identity matrix. The following identities hold for the four types of BD:
\begin{equation}
\label{equ:bd_identity}
\renewcommand{\arraystretch}{1.1}
\begin{array}{l}
\text{row-first: }
\vec{W} \equiv \begin{bmatrix}\vec{I} \\ \vec{C}\end{bmatrix} \vec{B},
\quad
\text{row-last: }
\vec{W} \equiv \begin{bmatrix}\vec{C} \\ \vec{I}\end{bmatrix} \vec{B},\\[4pt]
\text{col-first: }
\vec{W} \equiv \vec{B}[\vec{I}, \vec{C}],
\quad
\text{col-last: }
\vec{W} \equiv \vec{B}[\vec{C}, \vec{I}].
\end{array}
\end{equation}

While the theoretical reconstruction is exact, numerical residuals may arise in practice due to finite-precision arithmetic or ill-conditioned basis matrices. To mitigate this, we compare the reconstruction errors from the first- and last-\(r\) basis candidates and retain the one with the smaller Frobenius norm residual. The full procedure for \textit{row-based} Basis Decomposition is summarized in Algorithm~\ref{alg:bd} (BD decomposition) and Algorithm~\ref{alg:bd_reconstruct} (BD reconstruction), where a subset of rows is selected as the basis. The \textit{column-based} variant can be formulated analogously and is omitted here. An offset-search variant of Residual-min that further improves numerical stability is described in Appendix~\ref{sec:offset_search}.

By losslessly replacing standard low-rank matrix multiplication with a more compact and computationally efficient alternative, BD is broadly applicable to scenarios involving low-rank multiplications. This includes applications such as neural network inference and data compression.

% \subsection{Basis Decomposition for Artificial Neural Networks}

% Low-rank matrix multiplication has been widely used in modern neural network architectures to reduce parameter count and computational cost, particularly in large models such as Transformers \citep{vaswani2017attention}. Since BD operates directly on low-rank matrix products, it can be seamlessly applied to \textbf{losslessly} accelerate common neural network components where such products appear.

\subsection{BD for Linear Layer}
\label{sec:bd_linear}
The linear layer is the most common component in neural networks. Given an input vector \(\vec{x} \in \mathbb{R}^{d_{\text{in}}}\), a standard linear layer computes the output \(\vec{y} \in \mathbb{R}^{d_{\text{out}}}\) using a weight matrix \(\overline{\vec{W}} \in \mathbb{R}^{d_{\text{in}} \times d_{\text{out}}}\):
\begin{equation}
  \vec{y} = \vec{x} \overline{\vec{W}}.
\end{equation}
To reduce parameter count and computational cost, many recent works adopt low-rank approximations of the weight matrix. A common approach is to factorize \(\overline{\vec{W}}\) as \(\overline{\vec{W}} \approx \vec{U} \vec{V}^{\top}\), where \(\vec{U} \in \mathbb{R}^{d_{\text{in}} \times r}\) and \(\vec{V} \in \mathbb{R}^{d_{\text{out}} \times r}\) with \(r < \min(d_{\text{in}}, d_{\text{out}})\). The resulting low-rank layer becomes:
\begin{equation}
\vec{y} = (\vec{x} \vec{U}) \vec{V}^{\top},
\end{equation}
which reduces both the number of parameters from \(d_{\text{in}} d_{\text{out}}\) to \(r (d_{\text{in}} + d_{\text{out}})\).

Such low-rank structures appear in various domains in deep learning: (1)~\textbf{low-rank pruning}, where pretrained weight matrices are compressed post hoc via SVD-like approximations~\citep{hsu2022language,yuan2023asvd,wang2024svd,zhao2025pivoting,jaiswal2024WeLore,saha2024compressing,kaushal2023lord,sharma2023truth,qinsidobi,liu2025eorafinetuningfreecompensationcompressed,sakr2024espace,ren-zhu-2024-low,lin2024modegptmodulardecompositionlarge,hajimolahoseini2022strategies}; (2)~\textbf{low-rank training}, where the weight matrices are parameterized as low-rank products throughout training~\citep{khodak2021initialization,schotthofer2022low,kamalakara2022exploring,zhao2023inrank,savostianova2023robust}; (3)~\textbf{low-rank + sparse hybridization}, which combines sparsity and low-rank approximations for improved performance~\citep{pmlr-v202-li23ap,han2024sltrain,zhang2025oats}; and (4)~\textbf{LoRA-style fine-tuning}~\citep{hu2022lora,zhang2023adaptive,lialin2024relora,meng2024pissa,dora,zhang2023increloraincrementalparameterallocation}, where a low-rank adaptation is injected into frozen models for efficient parameter updates.

Since Basis Decomposition (BD) operates directly on the product \(\vec{U} \vec{V}^{\top}\), it can be seamlessly applied to all these cases in a lossless and hardware-efficient manner.

To replace the low-rank linear layer with a BD layer, we decompose the weight product \(\vec{W} = \vec{U} \vec{V}^\top\) using column-based BD. Let \(\vec{B} \in \mathbb{R}^{d_{\text{in}} \times r}\) be the first-\(r\) column and \(\vec{C} \in \mathbb{R}^{r \times (d_{\text{out}} - r)}\) the coefficient matrix. The BD layer computes the output in two steps:
\begin{equation}
\vec{h} \gets \vec{x} \vec{B}, \quad \vec{y} \gets [ \vec{h}, \vec{h} \vec{C}].
\end{equation}
The last-\(r\) version is similar. For any \(r < \min(d_{\text{in}}, d_{\text{out}})\), BD achieves strictly lower FLOPs and memory cost than the original low-rank layer, reduced by \(\frac{r}{d_{\text{in}} + d_{\text{out}}}\) relative to the original.

\begin{figure}[t]
  \centering
  % ----------- 左侧 ------------
% \textbf{Inference pseudocode comparison between standard MHA and BD Attention. Red operations highlight the structural difference.}\par\medskip

\begin{minipage}[t]{0.485\textwidth}
    \begin{algorithm}[H]       % H: 禁止自动漂浮，允许嵌套
      \caption{MHA Inference}
      \label{alg:mha_inference}
       \small
      \begin{algorithmic}[1]
        \INPUT Weight \(\vec{W}_q \in \mathbb{R}^{d \times n d_h}\), \(\vec{W}_k \in \mathbb{R}^{d \times n d_h}\), \(\vec{W}_v \in \mathbb{R}^{d \times n d_h}\), \(\vec{W}_o \in \mathbb{R}^{n d_h \times d}\); Input \(\vec{X} \in \mathbb{R}^{L \times d}\)
        \STATE \(\vec{Q} \gets \vec{X} \vec{W}_q\)
        \STATE \(\vec{K} \gets \vec{X} \vec{W}_k\)
        \STATE \(\vec{V} \gets \vec{X} \vec{W}_v\)
        \STATE \([\vec{Q}_1, \dots, \vec{Q}_n] \gets \vec{Q}, \; [\vec{K}_1, \dots, \vec{K}_n] \gets \vec{K}, \; [\vec{V}_1, \dots, \vec{V}_n] \gets \vec{V}\)
        \STATE \(\vec{O}_i \gets \text{softmax}(\frac{\vec{Q}_i \vec{K}_i^\top}{\sqrt{d_h}})\vec{V}_i\)
        \STATE \(\vec{Y} \gets [\vec{O}_1, \dots, \vec{O}_n] \vec{W}_o\)
        \OUTPUT \(\vec{Y}\)
      \end{algorithmic}
    \end{algorithm}
  \end{minipage}
  \hfill
  % ----------- 右侧 ------------
  \begin{minipage}[t]{0.485\linewidth}
    \begin{algorithm}[H]
      \caption{BD Attention Inference}
      \label{alg:mha_bd_inference}
      \small
      \begin{algorithmic}[1]
        \INPUT Weight \(\vec{B}_{qk} \in \mathbb{R}^{d \times n d_h}\), \(\vec{C}_{qk} \in \mathbb{R}^{\textcolor{red}{(d - d_h)} \times n d_h}\), \(\vec{C}_{vo} \in \mathbb{R}^{\textcolor{red}{(d - d_h)} \times n d_h}\), \(\vec{B}_{vo} \in \mathbb{R}^{n d_h \times d}\); Input \(\vec{X} \in \mathbb{R}^{L \times d}\)
        \STATE \(\vec{Q}^{\prime} \gets \vec{X} \vec{B}_{qk}\)
        \STATE \({\vec{K}^{\prime}} \gets \textcolor{red}{{[\vec{X}_{:,\,1:d_h}]}^{\times n} + \vec{X}_{:,\,d_h:d}\vec{C}_{qk}}\)
        \STATE \(\vec{V}^{\prime} \gets \textcolor{red}{{[\vec{X}_{:,\,1:d_h}]}^{\times n} + \vec{X}_{:,\,d_h:d}\vec{C}_{vo}}\)
        \STATE \([\vec{Q}_1^{\prime}, \dots, \vec{Q}_n^{\prime}] \gets \vec{Q}^{\prime}, \; [\vec{K}_1^{\prime}, \dots, \vec{K}_n^{\prime}] \gets \vec{K}^{\prime}, \; [\vec{V}_1^{\prime}, \dots, \vec{V}_n^{\prime}] \gets \vec{V}^{\prime}\)
        \STATE \(\vec{O}_i^{\prime} \gets \text{softmax}(\frac{\vec{Q}_i^{\prime} {\vec{K}_i^{\prime}}^\top}{\sqrt{d_h}})\vec{V}_i^{\prime}\)
        \STATE \(\vec{Y} \gets [\vec{O}_1^{\prime}, \dots, \vec{O}_n^{\prime}] \vec{B}_{vo}\)
        \OUTPUT \(\vec{Y}\)
      \end{algorithmic}
    \end{algorithm}
  \end{minipage}

\caption{Inference pseudocode comparison between standard MHA and BD Attention. \textcolor{red}{Red operations} highlight the structural difference.}
\label{fig:alg_compare}
\vspace{-1em}
\end{figure}

\subsection{BD for Multi-Head Attention}
Multi-head attention (MHA) is a central component in Transformer-based architectures, widely adopted in large language models (LLMs)~\citep{vaswani2017attention,radford2018improving,brown2020languagemodelsfewshotlearners,touvron2023llama} and vision-language models (VLMs)~\citep{dosovitskiy2021an,radford2021learning,li2022blip}. The comparison between BD Attention and MHA is illustrated in Figure \ref{fig:bd_illustration}.

We begin by reviewing the standard multi-head attention (MHA) mechanism. Let $d$ be the input (embedding) dimension, $n$ be the number of attention heads, $d_h$ be the dimension per head, $L$ be the input sequence length, and $\vec{X} \in \mathbb{R}^{L \times d}$ be the attention input. MHA produces queries, keys and values ($\vec{Q},\vec{K},\vec{V} \in \mathbb{R}^{L \times n d_h}$) by three projection matrices ($\vec{W}_q,\vec{W}_k,\vec{W}_v \in \mathbb{R}^{d \times n d_h}$):
\begin{equation}
\vec{Q} = \vec{X} \vec{W}_q, \quad \vec{K} = \vec{X} \vec{W}_k, \quad \vec{V} = \vec{X} \vec{W}_v
\end{equation}
\(\vec{Q},\vec{K},\vec{V}\) are divided into $n$ heads for MHA:
\begin{gather}
[\vec{Q}_1, \dots, \vec{Q}_n] = \vec{Q},\quad [\vec{K}_1, \dots, \vec{K}_n] = \vec{K},\quad [\vec{V}_1, \dots, \vec{V}_n] = \vec{V}, \\
\vec{O}_i = \mathrm{softmax}\!\left(\frac{\vec{Q}_i \vec{K}_i^\top}{\sqrt{d_h}}\right)\vec{V}_i, \quad \vec{Y} = [\vec{O}_1, \dots, \vec{O}_n] \vec{W}_o.
\end{gather}
where $\vec{Q}_i, \vec{K}_i, \vec{V}_i \in \mathbb{R}^{L \times d_h}$ represent the respective query, key, and value vectors for the $i$-th head, and \(\vec{W}_o \in \mathbb{R}^{n d_h \times d}\) represents the output projection matrix.

We reformulate multi-head attention (MHA).
\begin{equation}
  \begin{aligned}
    \vec{Y} &= \sum_{i=1}^n \vec{O}_i \vec{W}_o^i = \sum_{i=1}^n \text{softmax}(\frac{\vec{Q}_i \vec{K}_i^\top}{\sqrt{d_h}})\vec{V}_i \vec{W}_o^i \label{equ:head_simple} \\
&= \sum_{i=1}^n \text{softmax}\left(\frac{\vec{X} (\vec{W}_q^i {(\vec{W}_{k}^{\,i})}^{\!\top}) \vec{X}^\top}{\sqrt{d_h}}\right)\vec{X} (\vec{W}_v^i \vec{W}_o^i) \\
\vec{W}_q &\rightarrow [\vec{W}_q^1, \dots, \vec{W}_q^n], \quad \vec{W}_k \rightarrow [\vec{W}_k^1, \dots, \vec{W}_k^n], \\
    \vec{W}_v &\rightarrow [\vec{W}_v^1, \dots, \vec{W}_v^n], \quad
    \vec{W}_o \rightarrow \begin{bmatrix}
    \vec{W}_o^1 \\
    \vdots \\
    \vec{W}_o^n
    \end{bmatrix}
  \end{aligned}
\end{equation}
where $\vec{W}_q^i, \vec{W}_k^i, \vec{W}_v^i \in \mathbb{R}^{d \times d_h}$ denote the $i$-th \emph{vertical slice} of the corresponding weight matrices, and $\vec{W}_o^i \in \mathbb{R}^{d_h \times d}$ denotes the $i$-th \emph{horizontal slice} of $\vec{W}_o$. Both \((\vec{W}_q^i {\vec{W}_k^i}^\top)\) and \((\vec{W}_v^i \vec{W}_o^i) \) are a matrix with shape \(d \times d_h\) multiply a matrix with shape \(d_h \times d\). As \(d_h < d\), this reveals a key insight:

\begin{tcolorbox}[colback=gray!10, colframe=gray!50, sharp corners, boxrule=0.5mm, width=\columnwidth]
\textbf{\textit{Key Insight}}: Each head’s QK and VO computation are inherently low-rank matrix multiplications.
\end{tcolorbox}

We can apply Basis Decomposition (BD) \textbf{offline during model preparation} to compress these projection weight. Taking QK as an example (VO in Appendix \ref{sec:vo_mha}), we decompose the weight product \((\vec{W}_q^i {\vec{W}_k^i}^\top)\) using column-based BD. Let \(\vec{B}_{qk}^i \in \mathbb{R}^{d \times d_h}\) be the first-$r$ columns (last-$r$ is similar) and \(\vec{C}_{qk}^i \in \mathbb{R}^{d_h \times (d - d_h)}\) the coefficient matrix, we can convert the attention score into expression of BD:
\begin{equation}
\label{equ:qk_bd_single}
\begin{aligned}
    \text{attn\_score}_i &= \vec{Q}_i \vec{K}_i^\top = \vec{X} (\vec{W}_q^i {(\vec{W}_k^i)}^\top) \vec{X}^\top = \vec{X} (\vec{B}_{qk}^i [\vec{I}, \vec{C}_{qk}^i]) \vec{X}^\top \\
    &= (\vec{X} \vec{B}_{qk}^i) ([\vec{I}, \vec{C}_{qk}^i] \vec{X}^\top) = (\vec{X} \vec{B}_{qk}^i) (\vec{X}_{:,\,1:d_h}^\top + \vec{C}_{qk}^i\vec{X}_{:,\,d_h:d}^\top) \\
    &= \vec{Q}_i^{\prime} {\vec{K}_i^{\prime}}^\top \\
\vec{Q}_i^{\prime} &\leftarrow \vec{X} \vec{B}_{qk}^i, \quad {\vec{K}_i^{\prime}}^\top \leftarrow \vec{X}_{:,\,1:d_h}^\top + \vec{C}_{qk}^i\vec{X}_{:,\,d_h:d}^\top
\end{aligned}
\end{equation}
where \(\vec{X} \rightarrow \big[\, \vec{X}_{:,\,1:d_h},\vec{X}_{:,\,d_h:d} \,\big]\) partitions \(\vec{X}\) into its first $d_h$ and the remaining $d - d_h$ columns. By aligning all head's BD to either first-$r$ or last-$r$, we avoid calculating each head's Q and K projection separately, thereby reducing I/O overhead. The choice between first-\(r\) and last-\(r\) columns is determined by comparing the average residuals across all heads (Algorithm~\ref{alg:mha_bd_prepare}). This alignment allows Equation~\ref{equ:qk_bd_single}, which computes QK for a single head, to be reformulated to compute QK for all heads simultaneously:
\begin{equation}
\label{equ:qk_bd_merge}
\vec{Q}^{\prime} \leftarrow \vec{X} \vec{B}_{qk}, \quad \vec{K}^{\prime} \leftarrow {[\vec{X}_{:,\,1:d_h}]}^{\times n}
                  + \vec{X}_{:,\,d_h:d}\vec{C}_{qk}.
\end{equation}
where
\begin{equation}
\label{equ:qk_bd_merge_where}
\begin{aligned}
\vec{B}_{qk} \leftarrow [\vec{B}_{qk}^1, \dots, \vec{B}_{qk}^n], \quad \vec{C}_{qk} \leftarrow [{(\vec{C}_{qk}^1)}^\top, \dots, {(\vec{C}_{qk}^n)}^\top].
\end{aligned}
\end{equation}
The operator ${[\vec{X}_{:,\,1:d_h}]}^{\times n}$ denotes repeating the matrix
$n$ times along the second dimension.
The value-output (VO) projection can be processed analogously using the
row-based BD formulation (Appendix~\ref{sec:vo_mha}).

We compare the original MHA inference (Algorithm~\ref{alg:mha_inference})
with the BD-based MHA inference (Algorithm~\ref{alg:mha_bd_inference}).
The only differences (highlighted in \textcolor{red}{red}) lie in the
computation of keys $\vec{K}$ and values $\vec{V}$, where BD replaces the
original projections with smaller matrix multiplications.
A discussion of the interaction between positional embeddings and BD is
provided in Appendix~\ref{sec:pos_enc}.

\paragraph{Preservation of query–key similarity.}
The transformed projections $\vec{Q}'$ and $\vec{K}'$ satisfy $\vec{Q}'_i {\vec{K}'_i}^\top = \vec{Q}_i \vec{K}_i^\top$, which means that every pairwise inner product between queries and keys is exactly preserved.
$\vec{Q}'_i$ and $\vec{K}'_i$ can be regarded as an \emph{inner-product isomorphic} representation of the original $\vec{Q}_i,\vec{K}_i$ in a $d_h$-dimensional space.
We therefore still denote them as queries and keys ($\vec{Q}'_i, \vec{K}'_i$) to emphasize that they still maintain the essential property of attention: query–key similarity.
Advanced compression methods relying on query–key similarity, such as KV-cache compression, remain fully compatible with BDA.
Since the inner products are exactly preserved, these methods can be seamlessly integrated with BDA, enabling it to serve as a general and complementary acceleration framework.

\section{Experiments}
\label{sec:exp}

We evaluate BD Attention (BDA) from three perspectives: inference accuracy and efficiency, training evaluation, and integration with low-rank pruning.

\begin{figure}[t]
  \vspace{-1em}
    \centering
    \subfloat[\label{fig:bd_end2end_error} Perplexity increase]{
        \includegraphics[width=0.45\linewidth]{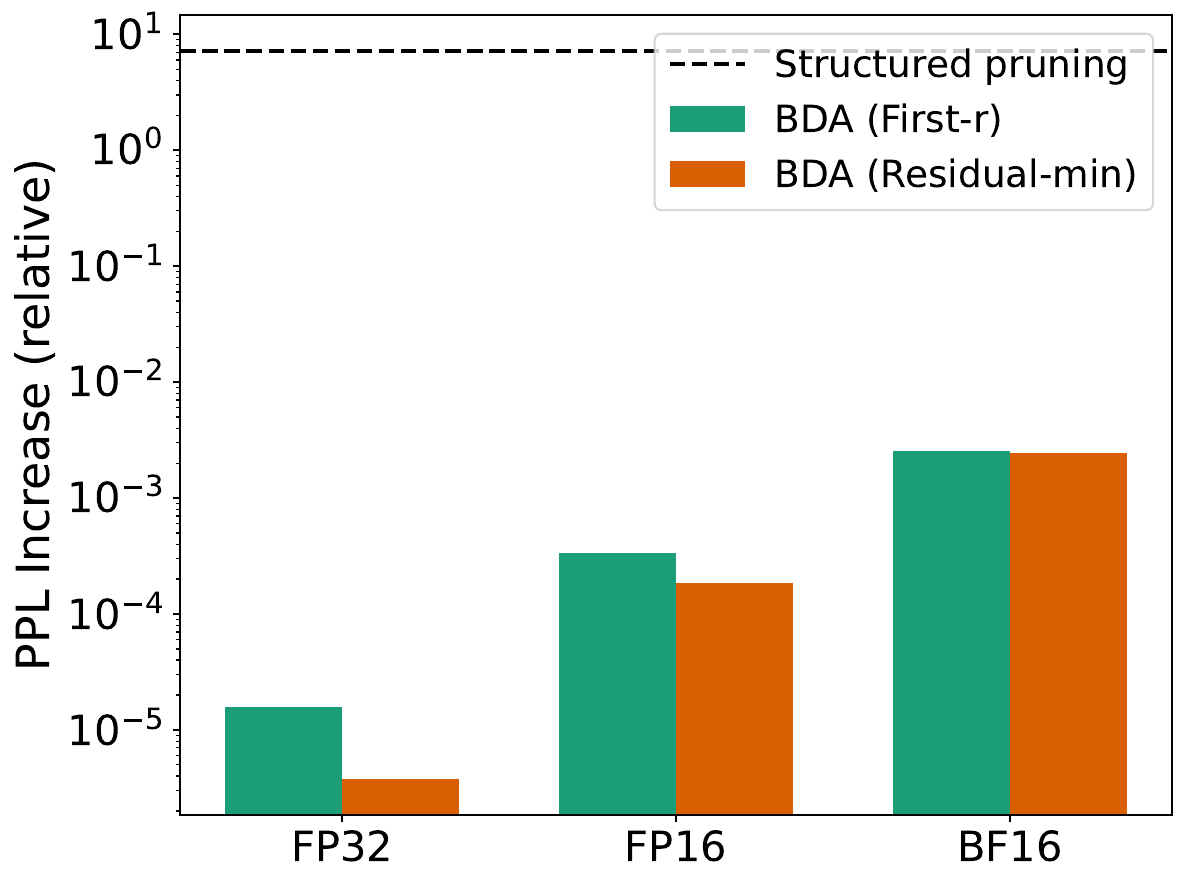}}
    \hfill
    \subfloat[\label{fig:bd_speedup} Efficiency]{
        \includegraphics[width=0.505\linewidth]{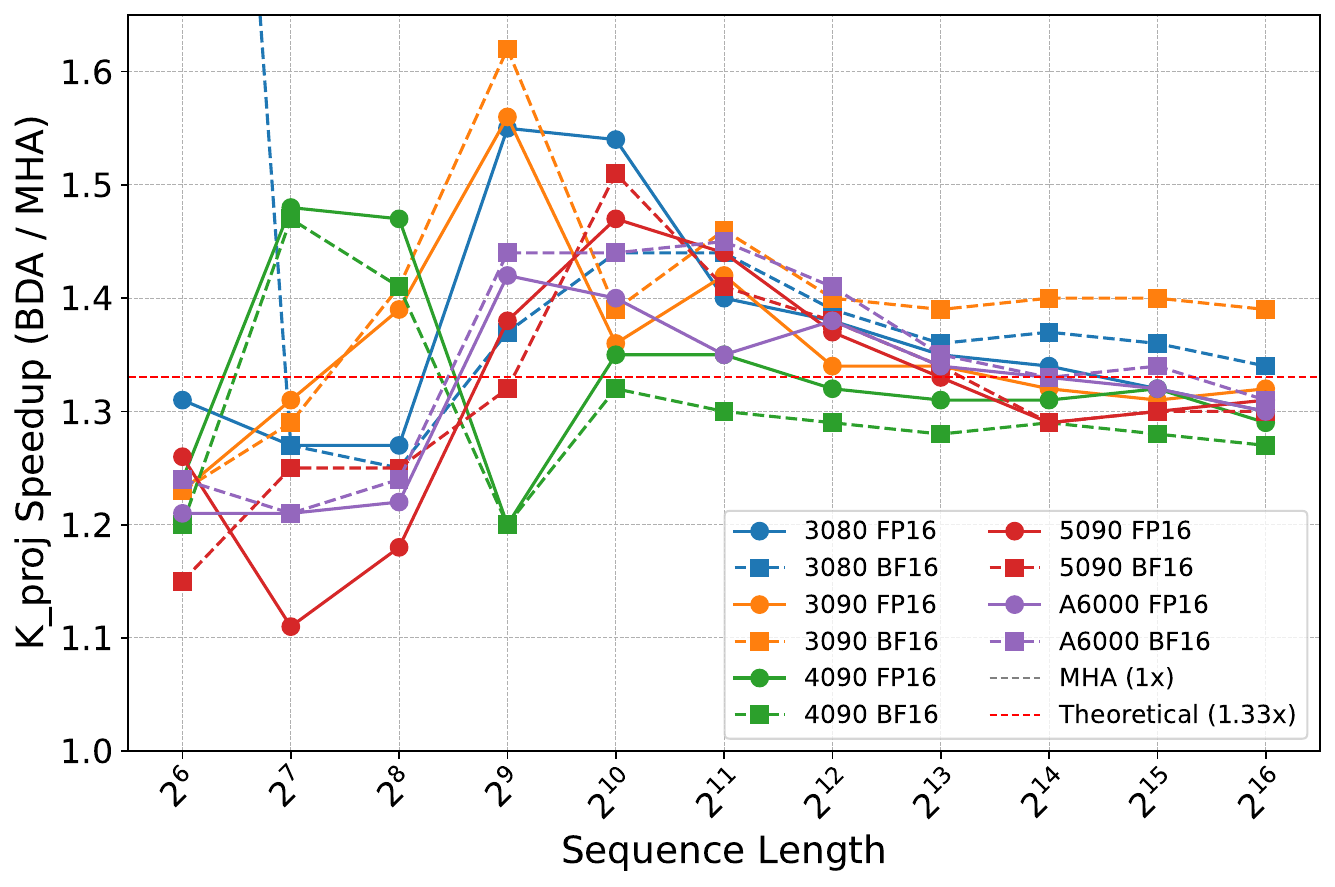}}
    \caption{\textbf{Evaluation of BD Attention (BDA).}
    \textbf{(a) Perplexity increase}: Perplexity ($\downarrow$) increase on WikiText2 when replacing all MHA layers of DeepSeek-V2-Lite with BDA.
    The increase is nearly imperceptible \textbf{0.02\% (FP16)}, with \textit{Residual-min} performing better. For reference, the dashed line shows the degradation from a structured pruning baseline at the same compression ratio ($25\%$ K/V channels).
    \textbf{(b) Efficiency}: Relative speedup of the $k\_proj$ operator under FP16 and BF16, measured across multiple NVIDIA architectures (Ampere, Ada, Blackwell).
    The dashed line at $1.33\times$ marks the theoretical bound. Speedups fluctuate around this line but remain consistent across GPUs, averaging \textbf{1.34$\times$} (FP16) and \textbf{1.37$\times$} (BF16).
    BDA also reduces parameter and memory usage by \textbf{25\%}.}
    \label{fig:bd_eval}
\vspace{-1em}
\end{figure}

% \begin{figure}[h]
%     \centering
%     \includegraphics[width=0.6\columnwidth]{plot/ppl_diff.pdf}
%     \caption{\textbf{End-to-end evaluation of BD Attention.}
%     Perplexity ($\downarrow$) increase on WikiText2 when replacing all MHA layers of DeepSeek-V2-Lite with BDA.
%     The increase is \textbf{nearly imperceptible} ($\sim$$10^{-4}$), with \textit{Residual-min} performing best.
%     For reference, the dashed line shows the degradation from a structured pruning baseline at the same compression ratio ($25\%$ K/V channels).}

%     \label{fig:bd_end2end_error}
% \end{figure}

\subsection{BDA Inference Accuracy}
\label{sec:exp_infer_acc}

Figure~\ref{fig:bd_end2end_error} reports the increase in perplexity on WikiText2 when replacing all MHA layers in DeepSeek-V2-Lite (16B) \citep{liu2024deepseekv2} with BDA.
Two strategies are compared: \textit{First-$r$}, which always selects the first $r$ rows, and \textit{Residual-min}, which adaptively selects between the first or last $r$ rows depending on the smaller reconstruction residual.

The perplexity increase is nearly imperceptible \textbf{0.0004\% (FP32), 0.02\% (FP16), 0.2\% (BF16)}, with \textit{Residual-min} consistently outperforming \textit{First-$r$}.
A similar advantage is observed at the operator level in per-layer reconstruction errors (see Appendix Table~\ref{tab:bd_reconstruct_error}), where \textit{Residual-min} achieves up to an order-of-magnitude lower error in FP32.

For reference, we also include a structured pruning baseline that removes $25\%$ of K/V channels at the same compression ratio ($d_h/d=128/512=25\%$).
This baseline follows the relative-importance scoring strategy of Zhang et al.~\citep{zhang2024plugandplay}, where each channel’s importance is estimated, summed, and the least important $25\%$ are pruned.
Although more recent structured pruning techniques \citep{ma2023llmpruner, ouderaa2024the, ashkboos2024slicegpt} achieve better performance, they typically require access to a calibration dataset, which is beyond the scope of this comparison.
Here, structured pruning is reported only as a reference for the scale of perplexity degradation at the same compression ratio.

We further report zero-shot downstream results in
Appendix~\ref{app:zeroshot}, which are consistent with the perplexity trends observed here.

% \begin{figure}[h]
% 		\centering

%   \includegraphics[width=0.5\columnwidth]{plot/speedup.pdf}

%   \caption{\textbf{Speedup of BDA} for the $k\_proj$ operator under FP16 and BF16 with the DeepSeek-V3 configuration~\citep{liu2024deepseek}.
%     The dashed line at $1.33\times$ marks the theoretical bound.
%     Measured speedups fluctuate around this line but consistently exceed the MHA baseline,
%     averaging \textbf{1.34$\times$} (FP16) and \textbf{1.37$\times$} (BF16).
%     BDA also reduces parameter and memory usage by \textbf{25\%}.}

% 		\label{fig:bd_speedup}

% \end{figure}

\subsection{BDA Inference Efficiency}

\label{sec:exp_infer_eff}

We first recall that BD reduces the FLOPs of the key projection from
$L d_h (2d-1)$ to $L d_h (2(d-d_h))$, implying a theoretical speedup of
$\frac{2d-1}{2(d-d_h)}\approx 1.33\times$ under the DeepSeek configuration
($d=512,d_h=128$).

To rigorously evaluate whether this theoretical gain translates into
practice, we consider two complementary regimes: \textbf{(i)} a \textit{standard PyTorch implementation} without any kernel fusion, and \textbf{(ii)} a \textit{fused Triton kernel} evaluated across multiple GPU architectures. Together, these settings demonstrate that BD’s efficiency does not rely on implementation-specific I/O optimizations.

\paragraph{CPU (standard PyTorch implementation).} We implemented BDA in standard PyTorch and benchmarked on CPU, without any kernel fusion or hardware-specific tuning. BDA achieves acceleration on average \textbf{1.17$\times$} in FP16 and \textbf{1.40$\times$} in BF16. Figure~\ref{fig:bd_eval_cpu} visualizes the speedup across sequence lengths.

\paragraph{GPUs (fused Triton kernel).} Because GPUs provide substantially higher peak FLOPs than CPUs, overall runtime becomes more sensitive to memory traffic. A naïve PyTorch implementation of BDA would allocate several intermediate tensors, forcing multiple round-trips to off-chip GPU memory (HBM), which is much slower than on-chip SRAM. To avoid this overhead, we develop a fused Triton kernel that executes BDA in a single load–compute–store pass, matching the memory-access pattern of the cuBLAS matmul used in MHA.

We evaluate this kernel across several NVIDIA architectures —
Ampere (A6000, RTX 3080/3090), Ada (RTX 4090), and Blackwell (RTX 5090) —
using the \textbf{same} kernel configuration without retuning. Across all
GPUs, BDA achieves near-theoretical acceleration (Fig.~\ref{fig:bd_speedup}), averaging \textbf{1.34$\times$} (FP16) and \textbf{1.37$\times$} (BF16), while also reducing parameters and memory by \textbf{25\%}. The consistency across architectures indicates that the performance gains arise from BD’s reduced arithmetic cost rather than device-specific tuning.

\begin{table}[t]
\centering
\caption{\textbf{BD applied to low-rank pruning} on LLaMA2 models (FP16).
BD further improves throughput and memory efficiency over low-rank pruning while preserving perplexity.
Best throughput and memory are highlighted in \textbf{bold}.}
\label{tab:llama_lowrank}

  % \small

\begin{tabular}{llccc}
\toprule
Model & Metric & Dense & Low rank 80\% & BD (from low-rank)  \\
\midrule
\multirow{4}{*}{LLaMA2-7B}
& Throughput (no kv cache)   & 338.23  & 368.90  & \textbf{422.58} \\
& Throughput (kv cache)      & 3726.31 & 4244.27 & \textbf{5285.60} \\
& Memory (GB)             & 12.55   & 10.21   & \textbf{8.52} \\
& PPL                     & 5.47    & 7.50    & 7.50 \\
\midrule
\multirow{4}{*}{LLaMA2-13B}
& Throughput (no kv cache)   & 181.15  & 201.50  & \textbf{238.51} \\
& Throughput (kv cache)      & 2345.99 & 2566.04 & \textbf{2856.81} \\
& Memory (GB)             & 24.36   & 19.58   & \textbf{16.35} \\
& PPL                     & 4.88    & 6.41    & 6.42 \\
\bottomrule
\end{tabular}

% \vspace{-1em}
\end{table}

\paragraph{Comparison to PIFA-style attention.} For completeness, we also implement a \textit{PIFA-style} variant (not described in original paper \citep{zhao2025pivoting}, but constructed by us for comparison): for each head $i$, \textit{PIFA-style} variant run QR with column pivoting on $\,\vec{W}_Q^i(\vec{W}_K^i)^\top\,$ to select its own basis. Because each head chooses a different basis, this forces per-head slicing and copying of $\vec{X}$, causing the memory traffic to increase by number of heads times. Consequently, PIFA-style attention is even slower than baseline MHA (Table~\ref{tab:fp16_A6000_throughput}, \ref{tab:bf16_A6000_throughput}). In contrast, BD aligns \textbf{all} heads to a shared contiguous basis, so $\vec{X}$ is sliced only once and reused across heads, dramatically reducing I/O compared with \textit{PIFA-style} variant.

% \paragraph{Summary.}
% Taken together, these results show that BD’s efficiency originates from
% the algebraic reduction in FLOPs rather than device-specific tuning.
% It improves efficiency in CPU using a standard PyTorch implementation, and it achieves near-theoretical gains on GPUs when provided with a fused implementation. Moreover, compared to PIFA-style attention, BD’s shared basis across heads avoids redundant copy.

\subsection{BD for Low-Rank Pruning}
\label{sec:exp_lowrank}

We evaluate BD when applied on top of models already compressed by low-rank pruning (Section \ref{sec:bd_linear}).

Table~\ref{tab:llama_lowrank} reports results on LLaMA2-7B and LLaMA2-13B under three settings:
(i) \textbf{Dense}, the original pretrained LLaMA2 model;
(ii) \textbf{Low-rank (80\% density)}, where weights are pruned into a low-rank structure following Zhao et al.~\citep{zhao2025pivoting};
and (iii) \textbf{BD (from low-rank)}, where the pruned low-rank weights are further transformed using Basis Decomposition (BD).
No retraining is performed in any case.

Results show that BD consistently improves efficiency over the low-rank baseline while preserving perplexity.
On average, BD increases throughput by \textbf{17.21\%} and reduces memory usage by \textbf{16.52\%} compared to low-rank models,
while keeping perplexity nearly unchanged.
This demonstrates that BD is complementary to existing compression techniques and can serve as a plug-in acceleration step for low-rank pruned models.

\subsection{BDA Training Evaluation}
\label{sec:exp_train}

In BDA training, the basis and coefficient matrices (\(\vec{B}_{qk}, \vec{C}_{qk}, \vec{B}_{vo}, \vec{C}_{vo}\)) replace the original projections as the learnable parameters, and gradients flow through them via standard autograd. No basis re-selection or BD recomputation is needed at any training step—the forward pass is identical to Algorithm~\ref{alg:mha_bd_inference}. While inference is mathematically lossless, training dynamics can still differ from MHA because gradient updates on the BD parameters are not guaranteed to be equivalent to those on the original projections.
To evaluate this, we trained Transformer models on the IWSLT’14 English-to-German \citep{cettolo-etal-2014-report} using either standard MHA or BDA as the attention module, both under the \textit{Noam} learning-rate schedule~\citep{vaswani2017attention}.
We swept across four learning-rate scales ($0.5, 1, 2, 4$).
As shown in Table~\ref{tab:train_bleu}, despite potential differences in optimization dynamics, the final BLEU scores of BDA are \textbf{comparable} to those of MHA across all settings.
All hyperparameters are detailed in Appendix~\ref{sec:train_hyper}, and were kept identical between MHA and BDA.
This highlights that BDA requires no hyperparameter search or tuning, and thus can be \textit{seamlessly integrated} into existing training pipelines without additional cost, while maintaining model quality.

\begin{table}[t]
\centering

\caption{\textbf{Training evaluation of BD Attention (BDA).}
BLEU scores (\(\uparrow\)) on the IWSLT’14 using the Transformer model.  The columns correspond to the \textbf{LR scale}, i.e., the multiplicative factor applied to the learning rate of \textit{Noam} schedule~\citep{vaswani2017attention}.  Across all scales, BDA achieves \textbf{comparable} BLEU scores to MHA, despite not guaranteeing identical gradients, and requires no hyperparameter tuning. \textbf{Bold} numbers indicate the higher BLEU for each LR scale.}

\label{tab:train_bleu}
% \small

\begin{tabular}{lcccc}
\toprule
 & LR Scale=0.5 & LR Scale=1 & LR Scale=2 & LR Scale=4 \\
\midrule
MHA & 24.98 & 23.98 & 23.86 & 24.04 \\
BDA & \textbf{25.27} & \textbf{25.04} & \textbf{23.89} & \textbf{24.14} \\
\bottomrule
\end{tabular}
\vspace{-1em}
\end{table}

\section{Conclusion and Discussion}
\label{sec:conclusion}

We presented \textbf{BD Attention (BDA)}, a novel \textit{lossless algorithmic reformulation} of multi-head attention.
By applying Basis Decomposition (BD), BDA restructures projection matrices into a compact form that eliminates redundant parameters and arithmetic operations while preserving exact outputs.
Our experiments confirmed its practical benefits: near-theoretical speedups in inference with reduced memory footprint, training performance comparable to MHA, and additional efficiency gains when applied to low-rank pruned models. Overall, BDA offers a mathematically exact and versatile foundation for attention acceleration that complements existing system-level methods. Future work includes evaluating BDA in large-scale language model pretraining to study its impact on training efficiency and model quality.

\bibliographystyle{unsrt}
\bibliography{example_paper}

%%%%%%%%%%%%%%%%%%%%%%%%%%%%%%%%%%%%%%%%%%%%%%%%%%%%%%%%%%%%%%%%%%%%%%%%%%%%%%%
%%%%%%%%%%%%%%%%%%%%%%%%%%%%%%%%%%%%%%%%%%%%%%%%%%%%%%%%%%%%%%%%%%%%%%%%%%%%%%%
% APPENDIX
%%%%%%%%%%%%%%%%%%%%%%%%%%%%%%%%%%%%%%%%%%%%%%%%%%%%%%%%%%%%%%%%%%%%%%%%%%%%%%%
%%%%%%%%%%%%%%%%%%%%%%%%%%%%%%%%%%%%%%%%%%%%%%%%%%%%%%%%%%%%%%%%%%%%%%%%%%%%%%%
\newpage
\appendix

\begin{algorithm}[ht]
\caption{BD Decomposition (Row)}
\label{alg:bd}
\begin{algorithmic}[1]
\INPUT \(\vec{W} \leftarrow \Vec{U}\Vec{V}^{\top}\), \(\vec{W} \in \mathbb{R}^{m \times n}\) with rank \(r\)
\STATE Extract first-\(r\) rows: \(\vec{B}_{F} \leftarrow\) first \(r\) rows of \(\vec{W}\)
\STATE Extract last-\(r\) rows: \(\vec{B}_{L} \leftarrow\) last \(r\) rows of \(\vec{W}\)

\STATE Solve coefficients:\par
\(\vec{C}_{F} \gets \mathrm{linsolve} (\vec{W}^{\prime} \setminus \vec{B}_{F} = \vec{C}_{F} \vec{B}_{F})\)\par
\(\vec{C}_{L} \gets \mathrm{linsolve} (\vec{W}^{\prime} \setminus \vec{B}_{L} = \vec{C}_{L} \vec{B}_{L})\)

\STATE Compute residuals\par
      \(R_{F} \gets \bigl\lVert \vec{W} -
        \begin{bmatrix}\vec{B}_{F} \\ \vec{C}_{F}\vec{B}_{F}\end{bmatrix}\bigr\rVert_F\)\par
      \(R_{L} \gets \bigl\lVert \vec{W} -
        \begin{bmatrix}\vec{C}_{L}\vec{B}_{L} \\ \vec{B}_{L}\end{bmatrix}\bigr\rVert_F\)
\STATE \textbf{Select better candidate}\par
      \textbf{if} \(R_{F}\le R_{L}\) \textbf{then}\par
      \quad \textit{tag}$\leftarrow$\,FIRST, \(\vec{B}\leftarrow\vec{B}_{\text{F}},\;\vec{C}\leftarrow\vec{C}_{F}\)\par
      \textbf{else}\par
      \quad \textit{tag}$\leftarrow$\,LAST, \(\vec{B}\leftarrow\vec{B}_{L},\;\vec{C}\leftarrow\vec{C}_{L}\)
\OUTPUT \textit{tag}, basis matrix \(\vec{B}\), coefficient matrix \(\vec{C}\)
\end{algorithmic}
\end{algorithm}

\begin{algorithm}[ht]
\caption{BD Reconstruction (Row)}
\label{alg:bd_reconstruct}
\begin{algorithmic}[1]
\INPUT
    \textit{tag}, basis matrix \(\vec{B}\in\mathbb{R}^{r\times n}\), coefficient matrix \(\vec{C}\in\mathbb{R}^{(m-r)\times r}\)

\STATE \textbf{if} \textit{tag} $=$ FIRST \textbf{then}
      \COMMENT{row-first in Equation \ref{equ:bd_identity}}
      \[
          \vec{W}\;\leftarrow\;
          \begin{bmatrix}
              \vec{B}\\
              \vec{C}\,\vec{B}
          \end{bmatrix}
      \]
\STATE \textbf{else} \COMMENT{row-last in Equation \ref{equ:bd_identity}}
      \[
          \vec{W}\;\leftarrow\;
          \begin{bmatrix}
              \vec{C}\,\vec{B}\\
              \vec{B}
          \end{bmatrix}
      \]
\OUTPUT Reconstructed matrix \(\vec{W}\in\mathbb{R}^{m\times n}\)
\end{algorithmic}
\end{algorithm}

\begin{algorithm}[h]
\caption{BD Attention Preparation (QK).
$R$ denotes the residual error of basis decomposition.}
\label{alg:mha_bd_prepare}
\begin{algorithmic}[1]
\INPUT \(\vec{W}_q, \vec{W}_k, \vec{W}_v, \vec{W}_o\) represent query, key, value and output projection matrix; $n$ be the number of attention heads
\FOR{\(i = 1, \ldots, n\)}

\STATE \textit{(first-$r$)}\;
\(R_{F}^i, \vec{B}_{F}^i, \vec{C}_{F}^i
\gets \mathrm{BD}^{\textsc{First}}_{\textsc{col}}(\vec{W}_q^i {\vec{W}_k^i}^\top)\)

\STATE \textit{(last-$r$)}\;
\(R_{L}^i, \vec{B}_{L}^i, \vec{C}_{L}^i
\gets \mathrm{BD}^{\textsc{Last}}_{\textsc{col}}(\vec{W}_q^i {\vec{W}_k^i}^\top)\)

\ENDFOR
\STATE Compute mean residuals: \(\bar{R}_{F} \gets \frac{1}{n}\sum_{i=1}^n R_{F}^i, \; \bar{R}_{L} \gets \frac{1}{n}\sum_{i=1}^n R_{L}^i\)
\STATE \textbf{Select better candidate}\par
      \textbf{if} \(\bar{R}_{F} \le \bar{R}_{L}\) \textbf{then}

      \quad \textit{tag}$\leftarrow$\,\textsc{First}, \(\vec{B}_{qk}^i \leftarrow \vec{B}_{F}^i,\;\vec{C}_{qk}^i \leftarrow \vec{C}_{F}^i, \; i = 1, \ldots, n\)

      \textbf{else}

      \quad \textit{tag}$\leftarrow$\,\textsc{Last}, \(\vec{B}_{qk}^i \leftarrow \vec{B}_{L}^i,\;\vec{C}_{qk}^i \leftarrow \vec{C}_{L}^i, \; i = 1, \ldots, n\)

\OUTPUT \textit{tag}, column basis matrices \(\vec{B}_{qk}^i\) (replacement of \(\vec{W}_q^i\)), coefficient matrix \(\vec{C}_{qk}^i\) (replacement of \(\vec{W}_k^i\)), \(i = 1, \ldots, n\)
\end{algorithmic}
\end{algorithm}

\begin{figure}[ht]
    \centering
    \includegraphics[width=0.8\columnwidth]{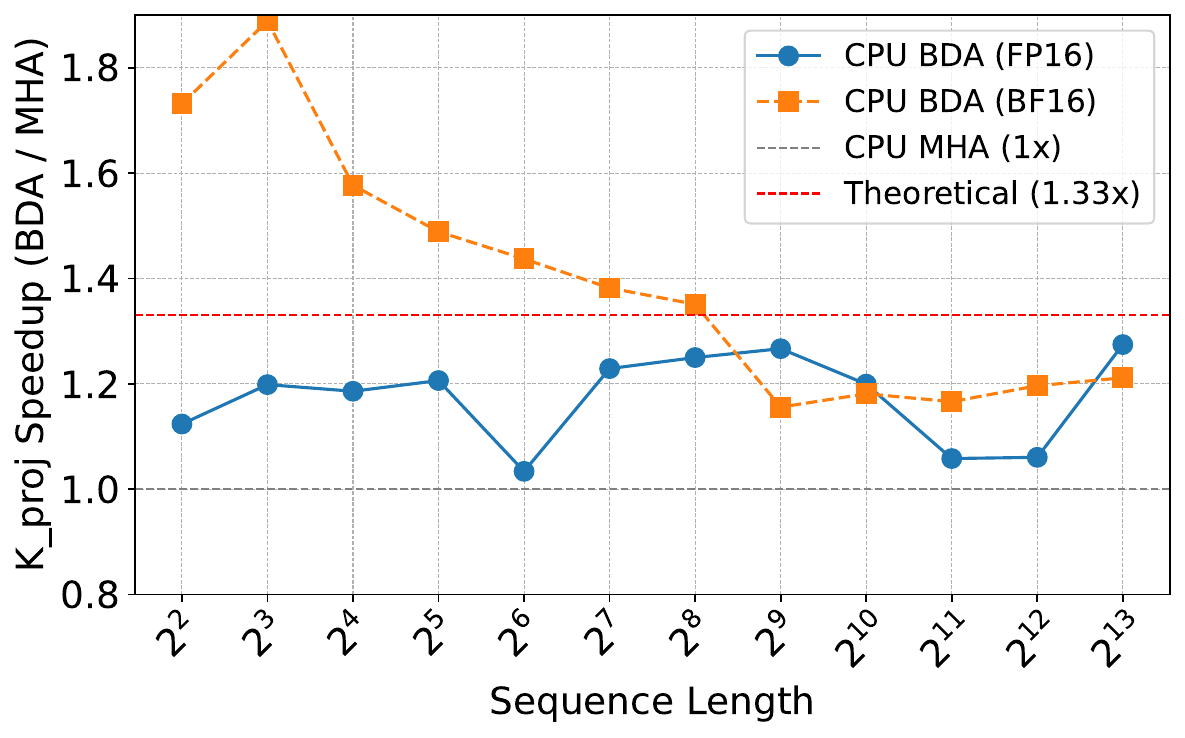}
    \caption{\textbf{BD Attention (Standard PyTorch Implementation) CPU Efficiency.}}
    \label{fig:bd_eval_cpu}
\end{figure}

\section{Proof of Almost Sure Full Rank of Random Matrices Theorem}
\label{sec:proof_full_rank_random}

\fullRankRandomTheorem*

\begin{proof}

\medskip
\noindent
\textbf{Step 1. Define the polynomial map $p(\vec{W}) = \det(\vec{W})$.}\\
Let
\[
p : \mathbb{R}^{r^2} \;\to\; \mathbb{R}, \quad
p(\vec{W}) \;=\; \det(\vec{W}),
\]
Note that $p$ is a polynomial (in $r^2$ variables)
and it is not the zero polynomial, since, for instance,
\[
p\bigl(\vec{I}_r\bigr) \;=\; \det(\vec{I}_r) \;=\; 1,
\]
where $\vec{I}_r$ denotes the $r\times r$ identity matrix.

\medskip
\noindent
\textbf{Step 2. The zero set of a nontrivial polynomial has measure zero.}\\
Consider
\[
Z \;=\; \{\, x \in \mathbb{R}^{r^2} \mid p(x) = 0 \}.
\]
Since $p$ is a nontrivial real polynomial, its zero set $Z$ is a real
algebraic variety of dimension $< r^2$, and hence we know
\[
\lambda(Z) \;=\; 0.
\]
(In simple terms, $Z$ is a ``lower-dimensional'' subset of $\mathbb{R}^{r^2}$.)

\medskip
\noindent
\textbf{Step 3. Absolute continuity of $\mu$ and conclusion.}\\
By hypothesis, the measure $\mu$ (the distribution of $\vec{W}$) is absolutely
continuous with respect to the Lebesgue measure $\lambda$, i.e.\ $\mu \ll \lambda$.
Hence for any Lebesgue null set $A$, we have $\mu(A) = 0$.
In particular, $Z$ has $\lambda(Z)=0$, so $\mu(Z)=0$.
But $Z$ exactly corresponds to
the event $\{\det(\vec{W})=0\}$ in the space of all matrix entries.
Therefore
\[
\Pr(\det(\vec{W})=0) \;=\; \mu(Z) \;=\; 0,
\]
which implies $\Pr(\det(\vec{W})\neq 0) = 1$, i.e.\ $\Pr(\mathrm{rank}(\vec{W})=r)=1$.
\end{proof}

\section{Relation to Interpolative Decomposition}
\label{sec:bd_vs_id}

BD is algebraically a special case of interpolative decomposition (ID) and CUR skeleton factorization (Equation~\ref{equ:bd_identity})~\citep{voronin2017efficient}. Its specialization to attention exploits three properties: (i) per-head $\vec{W}_q\vec{W}_k^\top$ and $\vec{W}_v\vec{W}_o$ are exactly rank-$d_h$ by construction, so BD reconstructs them with zero error—unlike ID-based weight compression~\citep{chee2022model} where targets are only approximately low-rank; (ii) the stochasticity of training induces a continuous distribution over the resulting weights, so Theorem~\ref{thm:full-rank-random} ensures any $d_h$ rows or columns form a valid basis with probability 1, letting BD avoid pivoted QR~\citep{businger1971linear}; (iii) BD selects \emph{contiguous} basis rows or columns as a hardware-friendly co-design, aligning with GPU memory-coalesced access and enabling a fused kernel that scattered pivot indices would prevent.

\section{BD for VO in MHA}
\label{sec:vo_mha}

Here, we introduce the BD transformation for value and output projection matrices. The right part of Equation \ref{equ:head_simple} can be converted to (using row-first in Equation \ref{equ:bd_identity}):

\begin{equation}
    \vec{V}_i \vec{W}_o^i = \vec{X} (\vec{W}_v^i \vec{W}_o^i) =  \vec{X} (\begin{bmatrix}\vec{I} \\ \vec{C}_{vo}^i\end{bmatrix} \vec{B}_{vo}^i) = ({\vec{X}_{:,\,1:d_h}} + \vec{X}_{:,\,d_h:d}\vec{C}_{vo}^i) \vec{B}_{vo}^i
\end{equation}

where \(\vec{B}_{vo}^i \in \mathbb{R}^{d_h \times d}\) be the first-$d_h$ row basis and \(\vec{C}_{vo}^i \in \mathbb{R}^{(d - d_h) \times d_h}\) the coefficient matrix. Similar to Equation \ref{equ:qk_bd_merge}, we can redefine value and output projection matrices to avoid calculating each head separately:

\begin{equation}
\begin{aligned}
    &\vec{V}^{\prime} = {[\vec{X}_{:,\,1:d_h}]}^{\times n} + \vec{X}_{:,\,d_h:d}\vec{C}_{vo} , \quad \text{where} \; \vec{C}_{vo} := [\vec{C}_{vo}^1, \dots, \vec{C}_{vo}^n] \\
    &[\vec{V}_1^{\prime}, \dots, \vec{V}_n^{\prime}] = \vec{V}^{\prime} \\
    &\vec{O}_i^{\prime} = \text{softmax}(\frac{\vec{Q}_i^{\prime} {\vec{K}_i^{\prime}}^\top}{\sqrt{d_h}})\vec{V}_i^{\prime} \\
    &\vec{Y} = [\vec{O}_1^{\prime}, \dots, \vec{O}_n^{\prime}] \vec{B}_{vo} , \quad \text{where} \; \vec{B}_{vo} := \begin{bmatrix}
    \vec{B}_{vo}^1 \\
    \vdots \\
    \vec{B}_{vo}^n
    \end{bmatrix}
\end{aligned}
\end{equation}

where \({[\vec{X}_{:,\,1:d_h}]}^{\times n}\) means repeat matrix $n$ times along second dimension. The \textit{last} version is similar.

\section{Hyperparameters for BDA Training}
\label{sec:train_hyper}

For the IWSLT’14 English-to-German task, we followed the standard Transformer setup with the \textit{Noam} learning-rate schedule~\citep{vaswani2017attention}.
The embedding dimension was set to 512, with a feed-forward dimension of 2048.
We used a batch size of 10,240 tokens and trained for 20,000 steps.
Dropout was 0.1 and label smoothing was 0.1.
The model contained 6 layers with 4 attention heads.
For simplicity, we removed positional embeddings inside the MHA module but retained them in the embedding layer (the effect of positional embeddings on BD is discussed in Appendix~\ref{sec:pos_enc}).
The learning rate scale was varied in $\{0.5, 1, 2, 4\}$ (Section~\ref{sec:exp_train}).
Training employed the Adam optimizer with 6,000 warmup steps.
Beam search with a beam size of 2 was used for evaluation, and BLEU scores were reported using the checkpoint with the lowest validation perplexity.

All hyperparameters for BDA were kept \textbf{identical} to those of MHA to ensure a fair comparison.
Moreover, this also shows that BDA naturally matches the training dynamics of MHA, achieving comparable performance without any hyperparameter tuning.
As a result, existing LLM training pipelines can migrate from MHA to BDA at essentially no cost.

\section{Effect of Positional Embedding on BD}
\label{sec:pos_enc}

We briefly discuss how positional embeddings interact with Basis Decomposition (BD).

\paragraph{Embedding-layer positional embedding.}
Any positional embedding applied at the embedding layer (e.g., learned positional embeddings or sinusoidal embeddings added to input tokens) does not affect BD. Since BD only restructures the projection matrices inside the attention, the addition of position-dependent vectors at the input embedding layer is orthogonal to BD’s reformulation.

\paragraph{MHA-internal positional embedding.}
When positional embeddings are applied inside the multi-head attention (MHA) module, the formulation changes. For example, vanilla Rotary Position Embedding (RoPE) \citep{su2021roformer} modifies the attention score computation as
\[
\vec{X}_n \vec{W}_q \, \vec{W}_k^\top  \vec{X}_m^\top
\quad \longrightarrow \quad
\vec{X}_n \vec{W}_q \, \vec{R}_{n-m} \, \vec{W}_k^\top  \vec{X}_m^\top
\]
where $\vec{R}_{n-m}$ is a rotation matrix depending on relative positions $n-m$.
BD guarantees the exact factorization
\[
\vec{W}_q \vec{W}_k^\top = \vec{B} [\vec{I}, \vec{C}],
\]
but in general BD cannot guarantee
\[
\vec{W}_q \vec{R}_{n-m} \vec{W}_k^\top = \vec{B} \vec{R}_{n-m} [\vec{I}, \vec{C}].
\]
Thus, vanilla RoPE breaks the exactness of BD.

\paragraph{Decoupled RoPE as a solution.}
DeepSeek proposes \textit{Decoupled RoPE} \citep{liu2024deepseekv2}, which splits attention channels into RoPE and non-RoPE parts. BD can then be applied to the non-RoPE channels, while RoPE channels remain unchanged. This is also the strategy used in our experiments.

\paragraph{Model-specific implications.}
\begin{itemize}
  \item \textbf{GPT models} use positional embeddings only at the input embedding layer; thus BD is fully lossless for both QK and VO.
  \item \textbf{LLaMA models} adopt vanilla RoPE inside MHA. Since RoPE is applied only to QK, BD remains lossless for VO projections but is not exact for QK.
  \item \textbf{DeepSeek models} employ Decoupled RoPE, which separates RoPE and non-RoPE channels in QK. BD can be applied losslessly to the non-RoPE channels of QK, and to all VO projections.
\end{itemize}

In summary, the compatibility of BD with positional embeddings depends on how positional information is integrated. Embedding-layer positional encodings pose no issue, while RoPE inside MHA requires modifications (e.g., Decoupled RoPE).

% \section{LLM Usage}
% Large Language Models (LLMs) were used solely as a writing assistant to polish grammar.
% They were not involved in research ideation, experiment design, analysis, or drafting of scientific content.

\begin{table}[ht]
\centering
\caption{\textbf{Numerical reconstruction errors of BD} for $\vec{W}_q \vec{W}_k^\top$ (QK) and $\vec{W}_v \vec{W}_o$ (VO) under different floating-point formats.
Values are averaged across all heads and all layers of the DeepSeek-V2-Lite model.
We compare two strategies for selecting the BD basis: (i) always using the first $r$ rows (\textit{First-$r$}), and (ii) choosing between the first or last $r$ rows depending on which yields the smaller residual (\textit{Residual-min}).
We report absolute mean squared error (MSE) and normalized mean squared error (NMSE).
The results confirm that BD introduces only \textbf{negligible} perturbations to the matrix products, with \textit{Residual-min} consistently outperforming \textit{First-$r$}, and improving error by at least one order of magnitude in FP32.}
\label{tab:bd_reconstruct_error}
\begin{tabular}{llccc}
\toprule
 & & FP32 & FP16 & BF16 \\
\midrule
\multirow{2}{*}{QK MSE}  & First-$r$ & $3.19\times10^{-12}$ & $1.09\times10^{-7}$ & $7.42\times10^{-7}$ \\
                         & Residual-min        & $3.12\times10^{-13}$ & $7.51\times10^{-8}$ & $6.51\times10^{-7}$ \\
\cmidrule(lr){1-5}
\multirow{2}{*}{QK NMSE} & First-$r$ & $5.74\times10^{-9}$  & $3.20\times10^{-4}$ & $2.07\times10^{-3}$ \\
                         & Residual-min        & $7.10\times10^{-10}$ & $2.36\times10^{-4}$ & $1.88\times10^{-3}$ \\
\cmidrule(lr){1-5}
\multirow{2}{*}{VO MSE}  & First-$r$ & $2.45\times10^{-12}$ & $1.02\times10^{-8}$ & $8.91\times10^{-8}$ \\
                         & Residual-min        & $2.15\times10^{-14}$ & $5.97\times10^{-9}$ & $8.19\times10^{-8}$ \\
\cmidrule(lr){1-5}
\multirow{2}{*}{VO NMSE} & First-$r$ & $1.26\times10^{-7}$  & $2.71\times10^{-4}$ & $2.28\times10^{-3}$ \\
                         & Residual-min        & $8.31\times10^{-10}$ & $1.61\times10^{-4}$ & $2.06\times10^{-3}$ \\
\bottomrule
\end{tabular}
\end{table}

\begin{table}[ht]
\centering
\caption{\textbf{Perplexity of BD Attention on WikiText2}.
All MHA layers in the DeepSeek-V2-Lite model are replaced with BD Attention, and we report perplexity ($\downarrow$) under different floating-point formats.
For BD, we show results using (i) always the first $r$ rows (\textit{First-$r$}) and (ii) selecting between the first or last $r$ rows based on the smaller residual (\textit{Residual-min}).
The last row reports the \textit{relative increase in perplexity (PPL)}, computed as $(\mathrm{PPL}_{\text{BD}} - \mathrm{PPL}_{\text{Original}})/\mathrm{PPL}_{\text{Original}}$.
Across all settings, BD introduces only \textbf{negligible} changes in model performance, with \textit{Residual-min} consistently yielding smaller increases.
Preparation time is very short (a few seconds), making BD efficient for deployment.}

\begin{tabular}{llccc}
\toprule
 &  & FP32 & FP16 & BF16 \\
\midrule
Original PPL &  & $6.306983$ & $6.307075$ & $6.310289$ \\
\cmidrule(lr){1-5}
\multirow{2}{*}{BD PPL}
  & First-$r$     & $6.307082$ & $6.309186$ & $6.326459$ \\
  & Residual-min  & $6.307007$ & $6.308252$ & $6.325656$ \\
\cmidrule(lr){1-5}
\multirow{2}{*}{PPL Increase (relative)}
  & First-$r$     & $0.002\%$ & $0.033\%$ & $0.256\%$ \\
  & Residual-min  & $0.0004\%$ & $0.019\%$ & $0.244\%$ \\
\cmidrule(lr){1-5}
\multirow{2}{*}{Preparation Time (s)}
  & First-$r$     & 3.56 & 1.89 & 2.39 \\
  & Residual-min  & 6.09 & 4.05 & 4.05 \\
\bottomrule
\end{tabular}
\end{table}

\section{Zero-shot Evaluation}
\label{app:zeroshot}

To provide a more comprehensive evaluation,
we additionally report zero-shot downstream performance
to complement the perplexity results in the main paper
and to assess whether the lossless property of BD
is preserved in downstream tasks.

\paragraph{Experimental Setup.}
We evaluate zero-shot performance on \textbf{DeepSeek-V2-Lite},
comparing the original MHA model (FP16) with its BD-based counterpart
(BDA, FP16).
All other settings are kept identical between MHA and BDA.
We follow the standard zero-shot evaluation protocol and report
mean accuracy and standard error, including ARC Easy and Challenge \citep{Clark2018ThinkYH}, SuperGLUE \citep{NEURIPS2019_4496bf24}, and HellaSwag \citep{zellers2019hellaswag}. Evaluations are conducted using the LM Evaluation Harness framework \citep{eval-harness}.

\begin{table}[ht]
\centering
\small
\caption{\textbf{Zero-shot accuracy} (mean $\pm$ std) on DeepSeek-V2-Lite.
Performance differences between MHA and BDA are consistently within
one standard error across all tasks.}
\label{tab:zeroshot}
\begin{tabular}{lcc}
\toprule
Task & MHA (FP16) & BDA (FP16) \\
\midrule
BoolQ          & 0.8021 $\pm$ 0.0070 & 0.8024 $\pm$ 0.0070 \\
MultiRC        & 0.5598 $\pm$ 0.0071 & 0.5611 $\pm$ 0.0071 \\
RTE            & 0.6065 $\pm$ 0.0294 & 0.6101 $\pm$ 0.0294 \\
WIC            & 0.4984 $\pm$ 0.0198 & 0.4984 $\pm$ 0.0198 \\
WSC            & 0.3654 $\pm$ 0.0474 & 0.3654 $\pm$ 0.0474 \\
HellaSwag      & 0.7779 $\pm$ 0.0041 & 0.7785 $\pm$ 0.0041 \\
ARC-Easy       & 0.7433 $\pm$ 0.0090 & 0.7412 $\pm$ 0.0090 \\
ARC-Challenge  & 0.4590 $\pm$ 0.0146 & 0.4599 $\pm$ 0.0146 \\
\bottomrule
\end{tabular}
\end{table}

\paragraph{Discussion.}
Across all evaluated tasks, the performance differences between MHA and
BDA remain well within one standard error.
This behavior is consistent with the theoretical guarantee that BD is a
lossless decomposition, as well as with the near-identical FP16
perplexity observed in language modeling.
These results confirm that BD preserves downstream zero-shot performance
in practice.

\begin{table}[t]
\centering
\caption{\textbf{Throughput comparison on FP16, NVIDIA A6000}.
We report throughput in million tokens per second (higher is better) for a single attention operator, comparing MHA, PIFA-style Attention, and BDA across different sequence lengths.
The tested matrix shape follows the DeepSeek-V3 configuration~\citep{liu2024deepseek}, with $n = 128$ heads, $d = 512$ (corresponding to $d_c$ in DeepSeek-V3), and $d_h = 128$ (same as $d_h$ in DeepSeek-V3).
The last column reports the \textit{relative speedup}, defined as BDA throughput divided by MHA throughput.}
\label{tab:fp16_A6000_throughput}
\begin{tabular}{rcccc}
\toprule
Seq. Len & MHA & PIFA-style (per-head QR) & BDA & Speedup \\
\midrule
64     & 1.79 & 0.99 & 2.16 & 1.21$\times$ \\
128    & 3.13 & 1.30 & 3.79 & 1.21$\times$ \\
256    & 4.46 & 1.52 & 5.43 & 1.22$\times$ \\
512    & 4.95 & 1.51 & 7.04 & 1.42$\times$ \\
1024   & 5.62 & 1.69 & 7.87 & 1.40$\times$ \\
2048   & 5.95 & 1.72 & 8.03 & 1.35$\times$ \\
4096   & 5.59 & 1.72 & 7.71 & 1.38$\times$ \\
8192   & 5.58 & 1.74 & 7.47 & 1.34$\times$ \\
16384  & 5.51 & 1.74 & 7.31 & 1.33$\times$ \\
32768  & 5.43 & 1.74 & 7.17 & 1.32$\times$ \\
65536  & 5.41 & 1.72 & 7.06 & 1.30$\times$ \\
\bottomrule
\end{tabular}
\end{table}

\begin{table}[t]
\centering
\caption{\textbf{Throughput comparison on BF16, NVIDIA A6000}. Setup and notation are the same as Table~\ref{tab:fp16_A6000_throughput}.}
\label{tab:bf16_A6000_throughput}
\begin{tabular}{rcccc}
\toprule
Seq. Len & MHA & PIFA-style (per-head QR) & BDA & Speedup \\
\midrule
64     & 1.74 & 0.98 & 2.16 & 1.24$\times$ \\
128    & 3.13 & 1.26 & 3.79 & 1.21$\times$ \\
256    & 4.46 & 1.52 & 5.56 & 1.24$\times$ \\
512    & 4.95 & 1.50 & 7.14 & 1.44$\times$ \\
1024   & 5.62 & 1.72 & 8.06 & 1.44$\times$ \\
2048   & 5.62 & 1.72 & 8.13 & 1.45$\times$ \\
4096   & 5.59 & 1.73 & 7.89 & 1.41$\times$ \\
8192   & 5.61 & 1.74 & 7.60 & 1.35$\times$ \\
16384  & 5.52 & 1.74 & 7.34 & 1.33$\times$ \\
32768  & 5.50 & 1.74 & 7.34 & 1.34$\times$ \\
65536  & 5.51 & 1.72 & 7.22 & 1.31$\times$ \\
\bottomrule
\end{tabular}
\end{table}

\section{Residual-Min with Offset Search}
\label{sec:offset_search}

Theorem~\ref{thm:full-rank-random} guarantees that any $d_h$ rows or columns form a valid basis with probability 1, but in finite precision the resulting $d_h\times d_h$ sub-blocks differ in conditioning: ill-conditioned blocks amplify rounding error during basis inversion. The Residual-min strategy in the main paper mitigates this by selecting between the first-$d_h$ and last-$d_h$ candidates based on Frobenius reconstruction residual.

We optionally extend this search to contiguous windows of size $d_h$ whose starting offsets are multiples of 16---a constraint that preserves the memory-coalesced access required for the fused kernel---and select the window with the smallest Frobenius residual. We refer to this variant as \textbf{Residual-min (offset search)}.

\paragraph{End-to-end perplexity.}
On DeepSeek-V2-Lite, replacing all MHA layers with BDA using Residual-min (offset search) reduces the FP16 PPL increase relative to the original model from $2\times10^{-4}$ (Residual-min between $\{$first, last$\}$) to $2\times10^{-5}$, an order-of-magnitude improvement.

\paragraph{Operator throughput cost.}
At the operator level (DeepSeek-V3 attention shape, FP16, A6000), Residual-min (offset search) averages $1.27\times$ MHA throughput versus $1.31\times$ for Residual-min $\{$first, last$\}$, i.e.\ a ${\sim}3\%$ throughput cost relative to the standard variant.

%%%%%%%%%%%%%%%%%%%%%%%%%%%%%%%%%%%%%%%%%%%%%%%%%%%%%%%%%%%%%%%%%%%%%%%%%%%%%%%
%%%%%%%%%%%%%%%%%%%%%%%%%%%%%%%%%%%%%%%%%%%%%%%%%%%%%%%%%%%%%%%%%%%%%%%%%%%%%%%

% The NeurIPS paper checklist is only required for the submission version.
% It is not included in the preprint version.
% \clearpage
% \input{checklist.tex}

\end{document}